%% file: main.tex
\documentclass[11pt]{article}
\usepackage{amsthm}
\usepackage{natbib}
\usepackage{graphicx} %
\usepackage{array} %
\usepackage{mathtools}
\usepackage{amsmath, amssymb, amsfonts, verbatim,bbm}
\usepackage{hyphenat,epsfig,subcaption,multirow}
\usepackage{nicefrac}
\usepackage{paralist}
\usepackage[utf8]{inputenc}
\usepackage[font=small, labelfont=bf]{caption}
\usepackage{wrapfig}
\usepackage[dvipsnames]{xcolor}
\usepackage{dsfont}
\newcommand{\Ind}{\mathds{1}}

\graphicspath{../}

\DeclareFontFamily{U}{mathx}{\hyphenchar\font45}
\DeclareFontShape{U}{mathx}{m}{n}{
      <5> <6> <7> <8> <9> <10>
      <10.95> <12> <14.4> <17.28> <20.74> <24.88>
      mathx10
      }{}
\DeclareSymbolFont{mathx}{U}{mathx}{m}{n}
\DeclareMathSymbol{\bigtimes}{1}{mathx}{"91}
\usepackage{tcolorbox}
\tcbuselibrary{skins,breakable}
\tcbset{enhanced jigsaw}

\usepackage[normalem]{ulem}
\usepackage[compact]{titlesec}

\definecolor{DarkRed}{rgb}{0.1,0.1,0.8}
\definecolor{DarkBlue}{rgb}{0.1,0.1,0.5}

\usepackage{nameref}
\definecolor{ForestGreen}{rgb}{0.1333,0.5451,0.1333}
\definecolor{DarkRed}{rgb}{0.8,0,0.4}
\definecolor{Red}{rgb}{0.8,0,0.4}
\usepackage[linktocpage=true,
    pagebackref=true,colorlinks,
    urlcolor=DarkRed,
    linkcolor=DarkRed, citecolor=ForestGreen,
    bookmarks,bookmarksopen,bookmarksnumbered]{hyperref}
\usepackage[noabbrev,nameinlink]{cleveref}
\creflabelformat{equation}{#2\textup{#1}#3}
\crefname{property}{property}{Property}
\creflabelformat{property}{(#1)#2#3}
\crefname{equation}{eq}{Eq}
\creflabelformat{equation}{#1#2#3}

\usepackage{bm}
\usepackage{url}
\usepackage{xspace}
\usepackage[mathscr]{euscript}

\usepackage{tikz}
\usetikzlibrary{arrows}
\usetikzlibrary{arrows.meta}
\usetikzlibrary{shapes}
\usetikzlibrary{backgrounds}
\usetikzlibrary{positioning}
\usetikzlibrary{decorations.markings}
\usetikzlibrary{patterns}
\usetikzlibrary{calc}
\usetikzlibrary{fit}
\usetikzlibrary{decorations.pathmorphing}

\usepackage{mdframed}

\usepackage[noend]{algpseudocode}
\makeatletter
\def\BState{\State\hskip-\ALG@thistlm}
\makeatother
\usepackage{enumitem}
\usepackage[margin=1in]{geometry}

\usepackage{soul}

\usepackage{array}
\usepackage{booktabs}
\usepackage{makecell}
\usepackage{tabularx}
\usepackage{multirow}
\usepackage{ragged2e}
\usepackage{colortbl}
\newcolumntype{P}[1]{>{\RaggedRight\hspace{0pt}}p{#1}}
\newcolumntype{L}{>{\arraybackslash}m{3cm}}
\newcolumntype{q}{>{\arraybackslash}m{4.5cm}}

\definecolor{aliceblue}{rgb}{0.94, 0.97, 1.0}
\definecolor{blizzardblue}{rgb}{0.67, 0.9, 0.93}

\newtheorem{lemma}{Lemma}[section]

\newtheorem{proposition}[lemma]{Proposition}
\newtheorem{theorem}[lemma]{Theorem}

\newtheorem{claim}[lemma]{Claim}

\newtheorem{definition}[lemma]{Definition}

\newtheorem*{claim*}{Claim}
\newtheorem*{proposition*}{Proposition}
\newtheorem*{lemma*}{Lemma}
\newtheorem*{problem*}{Problem}
\crefname{lemma}{Lemma}{Lemmas}
\crefname{note}{Note}{Notes}
\crefname{claim}{Claim}{Claims}
\newtheorem{remark}[lemma]{Remark}
\newtheorem{example}[lemma]{Example}
\newtheorem{openquestion}[lemma]{Open Question}

\newtheorem{mdresult}{Result}

\newtheoremstyle{restate}{}{}{\itshape}{}{\bfseries}{~(restated).}{.5em}{\thmnote{#3}}
\theoremstyle{restate}

\theoremstyle{definition}
\newtheorem{mdalg}{Algorithm}

\allowdisplaybreaks

\renewcommand{\qed}{\nobreak \ifvmode \relax \else
      \ifdim\lastskip<1.5em \hskip-\lastskip
      \hskip1.5em plus0em minus0.5em \fi \nobreak
      \vrule height0.75em width0.5em depth0.25em\fi}

\input{macros}

\title{\textbf{Proofs as Explanations: Short Certificates for Reliable Predictions}  
}

\author{
{\bf Avrim Blum}\thanks{Toyota Technological Institute at Chicago. \texttt{avrim@ttic.edu}} \quad
{\bf Steve Hanneke}\thanks{Purdue University. \texttt{steve.hanneke@gmail.com}} \quad
{\bf Chirag Pabbaraju}\thanks{Stanford University. \texttt{cpabbara@stanford.edu}} \quad
{\bf Donya Saless}\thanks{Toyota Technological Institute at Chicago. \texttt{donya@ttic.edu}}
}

\date{\today}

\begin{document}
\vspace{-3cm}
\maketitle
\vspace{-0.5cm}
{\renewcommand\thefootnote{}\footnotetext{Authors listed in alphabetical order.}%
}
\input{abstract}

\vspace{0.3cm}
\input{intro}
\input{preliminaries}
\input{worst-case}

\input{semi-distributional}

\input{distributionalarxiv}

\bibliographystyle{alpha}
\bibliography{general}

\clearpage

\appendix
\input{appendix}

\end{document}

%% file: macros.tex
\newcommand{\eps}{\ensuremath{\varepsilon}}

\newcommand{\R}{\ensuremath{\mathbb{R}}}

\newcommand{\poly}{\mbox{\rm poly}}

\DeclareMathOperator*{\Exp}{\ensuremath{{\mathbb{E}}}}
\DeclareMathOperator*{\Prob}{\ensuremath{\mathbb{P}}}
\renewcommand{\Pr}{\Prob}

\newcommand{\Ex}{\Exp}

\newenvironment{tbox}{\begin{tcolorbox}[
		enlarge top by=5pt,
		enlarge bottom by=5pt,
		 breakable,
		 boxsep=0pt,
                  left=4pt,
                  right=4pt,
                  top=10pt,
                  arc=0pt,
                  boxrule=1pt,toprule=1pt,
                  colback=white
                  ]%
	}
{\end{tcolorbox}}

\newcommand{\II}{\ensuremath{\mathbb{I}}}

\newcommand{\mireal}[1][]{
  \ifx\relax#1\relax%
    \II(\mione \,; \mitwo)%
  \else%
    \II(\mione \,; \mitwo\mid #1)%
  \fi
}

\newcommand{\cD}{\mathcal{D}}

\newcommand{\cH}{\mathcal{H}}

\usepackage{algorithm}
\usepackage{algpseudocode}



\usepackage[nameinlink]{cleveref}

\newcommand{\mcD}{\mathcal D}

\newcommand{\mcH}{\mathcal H}

\newcommand{\mcX}{\mathcal X}

\usepackage{dsfont}
\DeclareMathOperator{\cone}{Cone}

\newcommand{\hs}[1]{s_{#1}}

%% file: abstract.tex
\begin{abstract}
We consider a model for explainable AI in which an explanation for a prediction $h(x)=y$ consists of a subset $S'$ of the training data (if it exists) such that {\em all} classifiers $h' \in \cH$ that make at most  $b$ mistakes on $S'$ predict $h'(x)=y$.   Such a set $S'$ serves as a {\em proof} that $x$ indeed has label $y$ under the assumption that (1) the true target function $h^\star$ belongs to $\cH$, and (2) the set $S$ contains at most $b$ noisy or corrupted points.   For example, if $b=0$ and $\cH$ is the family of linear classifiers in $\mathbb{R}^d$, and if $x$ lies inside the convex hull of the positive data points in $S$ (and therefore every consistent linear classifier labels $x$ as positive), then Carath\'eodory's theorem states that $x$ in fact lies inside the convex hull of $d+1$ of those points.  So, a set $S'$ of size $d+1$ could be released as an explanation for a positive prediction, and would serve as a short proof of correctness of the prediction under the assumption of perfect realizability.  
In this work, we consider this problem more generally, for general hypothesis classes $\cH$ and general values $b\geq 0$.  We define the notion of the {\em robust hollow star number} of $\cH$ (which generalizes the standard hollow star number), and show that it precisely characterizes the worst-case size of the smallest certificate achievable, and analyze its size for natural classes. We also consider worst-case distributional bounds on certificate size, as well as {\em distribution-dependent} bounds that we show tightly control the sample size needed to get a certificate for any given test example.  In particular, we define a notion of the {\em certificate coefficient} $\eps_x$ of an example $x$ with respect to a data distribution $\cD$ and target function $h^\star$, and prove matching upper and lower bounds on sample size as a function of $\eps_x$, $b$, and the VC dimension $d$ of $\cH$.
\end{abstract}

%% file: intro.tex
\section{Introduction}

There has been substantial recent interest in {\em explainable} AI, \citep{ALI2023101805, 10.1145/3561048, doshivelez2017rigorousscienceinterpretablemachine, ribeiro2016whyitrustyou, miller2018explanationartificialintelligenceinsights}. For example, in a medical setting, if a classifier $h \in \cH$ trained on some large dataset $S$ predicts that patient $x$ should get treatment $y$, the patient's doctor may want an explanation of why.  Much of the work in explainable machine learning has focused on decision-tree models, or identifying the most salient features for the prediction made 
\citep{NIPS1995_45f31d16, Breiman1996BORNAT, Zhou2016InterpretingMV}.  In this work, we consider an alternative approach: outputting a subset $S'$ of the training set $S$ such that {\em all} classifiers $h' \in \cH$ that agree with $S'$ (or that make at most $b$ mistakes on $S'$) predict $h'(x)=y$, if such an $S'$ exists.  Such a set $S'$ would serve as a {\em proof} that $x$ indeed has label $y$ under the assumption that (1) the true target function $h^\star$ belongs to $\cH$, and (2) the set $S$ contains at most $b$ noisy or corrupted points.   For example, if $b=0$ and $\cH$ is the family of linear classifiers in $\mathbb{R}^d$, and if $x$ lies inside the convex hull of the positive data points in $S$ (and hence every consistent linear classifier labels $x$ as positive), then Carath\'eodory's theorem states that $x$ in fact lies inside the convex hull of $d+1$ of those points.  So, a set $S'$ of size $d+1$ could be released as an explanation for a positive prediction, and a proof of its correctness under the assumption of realizability. 

We aim to consider such explanations for general families $\cH$ and general values of $b$.
Our work is inspired by \cite{balcan2022robustly} who propose the notion of {\em robustly reliable} classifiers that, given an example $x$, output both a label $y$ and a value $b$ with the guarantee that any $h'\in\cH$ with $h'(x)\neq y$ makes strictly more than $b$ mistakes on $S$ (where $b<0$ if $x$ is not in the agreement region of the version space).  Our work can be viewed as investigating the shortest {\em proof} that can be provided for such a guarantee.

\subsection{Main Contributions}
\label{sec:main-contributions}
Our main contributions are the following:
\begin{enumerate}
    \item We formalize the notion of a robust certificate: a subset $S'$ of the training data that serves as a proof that a given example $x$ must have label $y$ if the target function belongs to a given class $\cH$ and the training set has at most $b$ noisy or corrupted points.
    To analyze this, we define the notion of the {\em robust hollow star number} of $\cH$, which generalizes the standard hollow star number \citep{bousquet2020proper}, and show that it precisely characterizes the worst-case size of the smallest certificate achievable for a class $\cH$, and analyze its size for natural classes.

    \item We examine worst-case distributional bounds on certificate size, showing that in this case, one can achieve tight bounds on certificates achievable from a finite sample in terms of the (standard) hollow star number of \cite{bousquet2020proper}.
    
    \item We also consider {\em distribution-dependent} bounds on the sample size needed to get a certificate for any given test example in terms of how ``close'' the example is to the ``boundary'' of the target function with respect to the  distribution $\cD$ and class $\cH$.  In particular, we define a notion of the {\em certificate coefficient} $\eps_x$ of an example $x$ with respect to a data distribution $\cD$ and target function $h^\star$, and prove matching upper and lower bounds on sample size as a function of $\eps_x$, $b$, and the VC-dimension $d$ of the class $\cH$.
    
    \item We examine how reweighted versions of the certificate coefficient can provide better bounds on the certificate size achievable from a polynomial-sized data sample.

\end{enumerate}

\subsection{Context and Related Work}
Explainable ML research largely focuses on decision trees or key predictive features. In these approaches, a certificate for an instance \( x \) corresponds to the root-to-leaf path of \( x \) in the tree. A widely studied method for explaining a black-box model \( f \) involves first learning a decision tree \( T \) that closely approximates \( f \). Once this surrogate decision tree \( T \) is obtained, a certificate for any instance \( x \) can be generated by retrieving its associated path in \( T \) \citep{NIPS1995_45f31d16, Breiman1996BORNAT, Zhou2016InterpretingMV}.
 \cite{blanc2021provablyefficientsuccinctprecise} proposed algorithms for implicitly learning the surrogate decision trees that approximate the target function, with provable performance guarantees under the uniform distribution.
 \cite{Ribeiro_Singh_Guestrin_2018} were the first to introduce certificates (which they term anchors) providing high precision by identifying a minimal set of rules that \emph{anchor} a prediction, ensuring that the output remains stable under small perturbations. They provide an efficient heuristic, based on greedy search, for finding such high-precision certificates. 
 \cite{blanc2022query} investigates the minimum number of queries required to certify a monotone function's prediction at a given point. They define as certificate a subset of input coordinates such that fixing these guarantees the function's value remains unchanged; \cite{gupta2022optimalalgorithmcertifyingmonotone} further investigate this.
 There is also an increasing interest in hybrid models that are partially interpretable as studied by \cite{ferry2023learninghybridinterpretablemodels, pmlr-v237-frost24a}.

Another related line of work considers learning guarantees in the face of malicious noise \citep{kearns1988learning,balcan2021noise}, such as the breakthrough work of \cite{awasthi2017power} on guarantees for learning linear separators under malicious noise, as well as work on instance-targeted data-poisoning attacks \citep{10.1145/1128817.1128824}. \cite{217486} and \cite{shafahi2018poisonfrogstargetedcleanlabel} demonstrated empirically that data-poisoning attacks can be highly effective, even when the adversary only adds correctly-labeled data to the training set. These targeted attacks have attracted attention in recent years due to their potential to compromise the trustworthiness of systems \citep{geiping2021witchesbrewindustrialscale, 6868201mozafari, chen2017targetedbackdoorattacksdeep}.

A key concern here is: when can predictions be trusted under such attacks? 
Most theoretical work on this question has focused on certifying stability of predictions under small changes to the training set \citep{gao2021learning,levine2021deep}.  However, recently,  \cite{balcan2022robustly, balcan2023reliablelearningchallengingenvironments, blum2024regularized} consider certifying the actual correctness of predictions, under assumptions of realizability and bounded adversarial power. 
Our work can be viewed as studying the shortest \textit{proof} that can be provided for such a guarantee.
 

%% file: preliminaries.tex
\section{Preliminaries}
\label{sec:preliminaries}

\subsection{Notation}
\label{sec:notation}

The input space is denoted \( \mathcal{X} \), and the label space \( \{+1, -1\} \). A hypothesis class $\mathcal{H}$ is a subset of $\{-1,+1\}^\mathcal{X}$. A sequence $S=\{(x_1,y_1),\dots,(x_n,y_n)\}$ is realizable by $\mcH$ if $\exists h \in \mcH$ such that $\sum_{i=1}^n\Ind[h(x_i)\neq y_i] = 0$. We use \( [n] \) to denote the first \( n \) natural numbers, $[n] = \{1, 2, \dots, n\}$.

\subsection{Formal Setting}
\label{sec:setting}

The primary subject in this work is the notion of a \textit{certificate} for predictions on test points. The certificates we consider are in terms of subsets of the training data. Concretely, suppose that we are given a dataset $S=\{(x_1,y_1),\ldots,(x_n,y_n)\}$ labeled by some unknown target function $h^\star \in \mcH$ (i.e., $y_i=h^\star(x_i), \forall i$).
An ideal scenario would be that there are no corruptions whatsoever in the labels, so that $S$ is completely realizable by $h^\star$. However, in practice, label corruptions are inevitable due to a variety of reasons like noisy measurements, human errors, etc. To account for this, we allow for a corruption budget $b \ge 0$. This leads to the following definition of a \textit{robustly realizable dataset}.

\begin{definition}[Robust Realizability]
    \label{def:robust-realizability}
    For a budget $b \ge 0$, a sequence $S=\{(x_1,y_1),\dots,(x_n,y_n)\}$ is $b$-robustly realizable by $\mcH$ if $\exists h \in \mcH$ such that $\sum_{i=1}^n\Ind[h(x_i)\neq y_i] \le b$.
\end{definition}

Given a $b$-robustly realizable dataset $S$, we wish to certify that the prediction on a given test point $x$ ought to be $y$, and we wish to frame this certificate in terms of a subset of $S$ that is ideally small. For this, we require the notion of an \textit{agreement} region.

\begin{definition}[Robust Agreement Region]
    \label{def:robust-agreement-region}
    A point $(x,y)$ is in the $b$-robust agreement region of a labeled sequence $S=\{(x_1,y_1),\dots,(x_n,y_n)\}$ if 
    \begin{align}
        \forall h \in \mcH: \sum_{i=1}^n\Ind[h(x_i)\neq y_i] \le b \implies h(x) = y.
    \end{align}
    When $b=0$, we refer to the $0$-robust agreement region simply as the agreement region.
\end{definition}
We are now ready to define our notion of robust certificates.

\begin{definition}[Robust Certificate]
    \label{def:robust-certificate}
    A sequence $S=\{(x_1,y_1),\dots,(x_n,y_n)\}$ is a $b$-robust certificate for $(x,y)$ if 
    \begin{enumerate}
        \item $S$ is $b$-robustly realizable by $\mcH$.
        \item $(x,y)$ is in the $b$-robust agreement region of $S$.
    \end{enumerate}
    $S$ is furthermore a minimal $b$-robust certificate for $(x,y)$ if $S$ is a $b$-robust certificate for $(x,y)$, and no proper subsequence $S' \subset S$ is a $b$-robust certificate for $(x,y)$.
\end{definition}

Our setting assumes that we are given a $b$-robustly realizable training dataset $S$, together with a test point $x$ which satisfies that $(x,y)$ belongs to the $b$-robust agreement region of $S$, for some $y \in \{+1,-1\}$. Namely, $S$ is itself a $b$-robust certificate for $(x,y)$.\footnote{The problem of obtaining small-size subsets of the training data that can serve as good-faith certificates for test-time predictions only makes sense if the test point is in the agreement region of the training data to start with.} Our primary objective is to analyze and obtain the smallest $S' \subseteq S$ such that $S'$ continues to be a $b$-robust certificate for $(x,y)$.

%% file: worst-case.tex
\section{Worst Case Bounds on Certificate Size}
\label{sec:worst-case}

The motivating question for this section is: what is the smallest certificate for $x$ that we can extract from $S$, and can a certificate of such size always be extracted, even for \textit{worst-case} instances of $S$ and $x$? We begin with two simple examples for certification  in the case of no corruptions (i.e., $b=0$).

\begin{example}[Halfspaces]
    \label{example:halfspaces}
    Consider the class of $d$-dimensional halfspaces passing through the origin, i.e., $\mcX = \R^d$, $\mcH = \{x \mapsto \operatorname{sign}(w^Tx):w \in \R^d\}$. Suppose we are given a training dataset $S=\{(x_1,y_1),\ldots,(x_n,y_n)\}$ realizable by $\mcH$. Let $S_{+}=\{x_i: i \in [n], y_i = +1\}$, $S_-=\{-x_i:i \in [n], y_i=-1\}$, and let $S_+ \cup S_-=\{z_1,\ldots,z_n\}$ for convenience. Note that we negate a negatively labeled $x_i$ in $S$ to form $S_-$. A test point $(x,+1)$ belongs to the agreement region of $S$ if and only if $x \in \cone(S_+ \cup S_-)$, where $\cone(\cdot)$ denotes the conic hull. To see this, suppose first that $x \in \cone(S_+ \cup S_-)$. This means that $x=\sum_{i=1}^n \alpha_i z_i$, where $\alpha_i \ge 0, \forall i$. Let $w \in \R^d$ represent any halfspace that labels all the examples in $S$ correctly. Then, observe that for every $z_i \in S_+ \cup S_-$, $w^Tz_i \ge 0$. But this also means that $w^Tx = \sum_{i=1}^n\alpha_i \cdot w^Tz_i \ge 0$. Thus, $(x, +1)$ belongs to the agreement region of $S$. 
    In the other direction, suppose
    that $x \notin \cone(S_+ \cup S_-)$. Then, because $\cone(S_+ \cup S_-)$ is a closed convex set, the separating hyperplane theorem (e.g., see Theorem 1 in \cite{separatinghyperplane}) implies the existence of $w \in \R^d$, such that $w^Tx < 0$, and $w^Ty \ge 0$ for $y \in \cone(S_+ \cup S_-)$. In particular, 
    this means that $w$ labels all examples in $S$ correctly. 
    However, $w^T x < 0$, so $(x, +1)$ is not in the agreement region of $S$.
    
    So, consider a test point $(x, +1)$ that belongs to the agreement region of $S$. By the preceding argument, $x \in \cone(S_+ \cup S_-)$. But then, by Carath\'eodory's Theorem, $x$ can be written as a conic combination of at most $d$ points\footnote{There are hardness results on computing the ``Carath\'eodory number'' of a point set \citep{bereg2020computing}.} from $S_+ \cup S_-$, which implies that $(x, +1)$ is in the agreement region of these points. Thus, we can use this subset of $S_+ \cup S_-$ (together with labels), which has size at most $d$, as a certificate for $(x, +1)$.
\end{example}

Note that in the above example, the size of the training data could be much larger than the ambient dimension, i.e., $n \gg d$. Even so, as long as the test point belongs to the agreement region of the data, it is possible to obtain a certificate of size at most $d$. One might observe that the VC dimension of $d$-dimensional halfspaces passing through the origin is $d$. Given how predominantly the VC dimension features in the characterization of learning-theoretic properties of binary hypothesis classes, one might wonder if the VC dimension also characterizes certificate size. That is, could it be possible to always extract an $O(VC(\mcH))$-size certificate from the training data, whenever the given test point belongs to its agreement region? Our next example shows that this is not the case.

\begin{example}[Singletons]
    \label{example:singletons}
    Consider the class of singletons on a domain of size $n$, i.e., $\mcX = [n]$, $\mcH = \{x \mapsto (-1)^{1-\Ind[x=i]}: i \in [n]\}$. The VC dimension of this class is 1. Suppose that we are given a training dataset $S=\{(1,-1), (2,-1),\ldots,(n-1, -1)\}$. Observe that the test point $(n, +1)$ is in the agreement region of $S$, because the only hypothesis in $\mcH$ that labels $S$ in the given way labels $n$ positively. Note however, that the test point is not in the agreement region of any proper subset of $S$. We must therefore necessarily provide all the $n-1$ points in $S$ to certify a positive label on $n$.
\end{example}
    
The instance in the above example has a curious property: as long as there remains a single point in the domain whose label we haven't observed, it is not possible to completely determine the label of the test point. 
This, however, is the defining property of another combinatorial quantity in learning theory, known as the \textit{hollow star number} \citep{bousquet2020proper}. For example, the hollow star number is known to lower bound the sample complexity of proper PAC learners \citep[Theorem 10]{bousquet2020proper}. 
Indeed, as we show ahead, the hollow star number ends up also being the relevant quantity that characterizes worst-case minimum certificate sizes. First, we define a slightly more generalized version of the hollow star number, which allows us to handle $b \ge 0$ corruptions.
\begin{definition}[Robust Realizability with Weights]
    \label{def:robust-realizability-wrt-weights}
    A weighted sequence $S=\{(x_i,y_i,w_i)\}_{i=1}^n$ is $b$-robustly realizable by $\mcH$ if 
    $\exists h \in \mcH \text{ such that }\sum_{i=1}^nw_i \cdot \Ind[h(x_i)\neq y_i] \le b.$
\end{definition}
\begin{definition}[Robust Hollow Star]
    \label{def:robust-hollow-star}
    A sequence $T=\{(x_1,y_1),\dots,(x_{\hs{b}},y_{\hs{b}})\}$ is a $b$-robust hollow star for $\mcH$, if there exist integer weights $w_1,\dots,w_{\hs{b}}$ and $i^* \in [\hs{b}]$ such that:
    \begin{enumerate}
        \item $w_i = 1$ for all $i \in [\hs{b}] \setminus \{i^*\}$.
        \item $w_{i^*} = b+1$.
        \item The weighted sequence $T=\{(x_i,y_i,w_i)\}_{i=1}^{s_b}$ is not $b$-robustly realizable by $\mcH$.
        \item For any $i \in [\hs{b}]$, $T \setminus \{(x_i,y_i,w_i)\}$ is $b$-robustly realizable by $\mcH$.
    \end{enumerate}
    The size $\hs{b}$ of the largest $b$-robust hollow star is the $b$-robust hollow star number of $\mcH$.
\end{definition}

\begin{remark}
    We observe that setting $b=0$ in the above definition recovers the standard definition (e.g., Definition 3 in \cite{bousquet2020proper}) of the hollow star number. We also note that repeats are allowed in the sequence $T$ (i.e., there might be $i \neq j$ where $x_i=x_j$).
\end{remark}

The following claim (proof in \Cref{sec:proof-b-robust-lower-bounded-by-hollow-star}) lower bounds the $b$-robust hollow star number in terms of the 0-robust hollow star number (referred to hereon simply as the hollow star number). 

\begin{claim}
    \label{claim:b-robust-lower-bounded-by-hollow-star}
    For $b \ge 0$, let $\hs{b}$ be the $b$-robust hollow star number of $\mcH$. Then, $\hs{b} \ge (b+1)(s_0-1)+1$.
\end{claim}
The $b$-robust hollow star number \textit{exactly} characterizes the smallest size of a reliable certificate.

\begin{theorem}[Robust Hollow Star Characterizes Minimum Certificate Size]
    \label{theorem:b-robust-hollow-star}
    Let $\mcH$ be a hypothesis class that has $b$-robust hollow star number $\hs{b}$, $S$ be a training dataset that is $b$-robustly realizable by $\mcH$, and $x$ be a test point such that for some $y \in \{-1,1\}$, $(x,y)$ is in the $b$-robust agreement region of $S$. Then, there exists a $b$-robust certificate $S' \subseteq S$ for $(x,y)$ that has size at most $\hs{b}-1$. Furthermore, there exists a training dataset $S$ of size $\hs{b}-1$ that is $b$-robustly realizable by $\mcH$, test point $x$ and test label $y \in \{-1,1\}$ which satisfy that $(x,y)$ belongs to the $b$-robust agreement region of $S$, such that no proper subsequence of $S$ is a $b$-robust certificate for $(x,y)$.
\end{theorem}

\begin{proof}
    We establish the upper and lower bound in order:
    
    \paragraph{Upper Bound.} Consider $S' \subseteq S$ to be the smallest subset of $S$ that is a $b$-robust certificate for $(x,y)$---let $S' =\{(x_1,y_1),\dots,(x_k,y_k)\}$. Note that $S'$ is a minimal $b$-robust certificate for $(x,y)$. We argue that $\{(x_1,y_1),\dots,(x_k,y_k),(x,1-y)\}$ is a $b$-robust hollow star, which implies $k+1 \le s_b$.
    
    Consider the weighted sequence $T=\{(x_1,y_1,1),\dots, (x_k,y_k,1),(x,1-y, b+1)\}$. Now consider any $h \in \mcH$. If $h(x)=y$, then the weighted error of $h$ on $x$ is $b+1$, and $h$ does not $b$-robustly realize $T$. Otherwise, we have that $h(x)=1-y$. In this case, it must be the case that $\sum_{i=1}^k \Ind[h(x_i) \neq y_i] > b$, otherwise $S'$ would not be a $b$-robust certificate for $(x,y)$. Summarily, the weighted sequence $T$ is not $b$-robustly realizable by any $h \in \mcH$.
    
    Now, consider $T' = T \setminus \{(x,1-y,b+1)\}$. Since $S$ is $b$-robustly realizable by $\mcH$, so is $S'$, which implies that the weighted sequence $T'$ is also $b$-robustly realizable by $\mcH$. Now, consider $T' = T \setminus \{(x_i,y_i,1)\}$ for any $i \in [k]$. Then, since $S'$ is a minimal $b$-robust certificate for $(x,y)$, there must exist some $h \in \mcH$ such that $\sum_{j=1, j \neq i}^k \Ind[h(x_j)\neq y_j] \le b$ and $h(x)=1-y$. Otherwise, $S' \setminus \{(x_i,y_i)\}$ would be a smaller $b$-robust certificate for $(x,y)$. Consequently, such an $h$ $b$-robustly realizes the weighted sequence $T'$, and we are done.

    \paragraph{Lower Bound.} Let $\{(x_1,y_1),\dots,(x_{\hs{b}},y_{\hs{b}})\}$ be the largest $b$-robust hollow star for $\mcH$, with corresponding weights $w_1,\dots,w_{\hs{b}}$. Let $i^*$ be such that $w_{i^*} = b+1$. We will set the training dataset $S=\{(x_i,y_i):i \neq i^*\}$, test point to be $x_{i^\star}$ and test label to be $1-y_{i^\star}$, and argue that $S$ is a minimal $b$-robust certificate for $(x_{i^\star}, 1-y_{i^\star})$. 
    
    Let $T$ be the weighted sequence $\{(x_1,y_1,w_1),\dots,(x_{\hs{b}},y_{\hs{b}},w_{\hs{b}})\}$. By definition of the $b$-robust hollow star, $T'=T \setminus \{(x_{i^*}, y_{i^*}, w_{i^*})\}$ is $b$-robustly realizable by $\mcH$. But note that all the weights in $T'$ are equal to $1$. This means that there exists $h \in \mcH$ such that $\sum_{i \in [\hs{b}], i \neq i^*}\Ind[h(x_i) \neq y_i] \le b$. Namely, $S$ is $b$-robustly realizable by $\mcH$. 
    
    Next, consider any $h \in \mcH$ that satisfies $\sum_{i \in [\hs{b}], i \neq i^*}\Ind[h(x_i) \neq y_i] \le b$. Then, it must be the case that $h(x_{i^*})=1-y_{i^*}$; otherwise, the weighted sequence 
    $$T=\{(x_1,y_1,w_1),\dots, (x_{i^*}, y_{i^*}, w_{i^*}),\dots,(x_{\hs{b}},y_{\hs{b}},w_{\hs{b}})\}$$ 
    would be $b$-robustly realizable by $\mcH$, which would contradict that $\{(x_1,y_1),\dots,(x_{\hs{b}},y_{\hs{b}})\}$ is a $b$-robust hollow star. Thus, $S$ is a $b$-robust certificate for $(x_{i^*}, 1-y_{i^*})$.
    
    Finally, we argue that $S$ is minimal. Consider any $S' \subsetneq S$. From above, we know $\exists h \in \mcH$ such that $S'$ is $b$-robustly realizable by $h$ and $h(x_{i^*})=1-y_{i^*}$. Furthermore, observe that $S' \cup \{(x_{i^*}, y_{i^*})\}$ excludes at least one of the members of the hollow star $S$. Let $T'$ be the weighted version of $S'$, and consider $T' \cup \{(x_{i^*}, y_{i^*}, w_{i^*})\}$. By the requirements of the $b$-robust hollow star, $T' \cup \{(x_{i^*}, y_{i^*}, w_{i^*})\}$ must be $b$-robustly realizable by $\mcH$. Because $w_{i^*}=b+1$, this means that $\exists h \in \mcH$ which $b$-robustly realizes $S'$, and satisfies $h(x_{i^*})=y_*$. Thus, $S'$ is not a $b$-robust certificate for $(x_{i^*}, 1-y_{i^*})$.
\end{proof}

\begin{remark}
    Let us revisit the examples of halfspaces and singletons considered earlier for the setting with no corruptions. For $d$-dimensional halfspaces passing through the origin (\Cref{example:halfspaces}), our certificate based on Carath\'eodory's theorem had size $d$. Indeed, this class has hollow star number equal to $d+1$, which means that this certificate size is optimal in general. On the other hand, the class of singletons on a domain of size $n$ has hollow star number equal to $n$. This validates why we were unable to obtain a certificate of size smaller than $n-1$ in \Cref{example:singletons}. 
\end{remark}

Thus, we have shown that in general, the $b$-robust hollow star number optimally characterizes minimum reliable certificate size. While \Cref{claim:b-robust-lower-bounded-by-hollow-star} relates the $b$-robust hollow star to the hollow star number (which is known for many natural classes, e.g., see the examples in Section 2.1 in \cite{bousquet2020proper}), it merely gives a lower bound. It would be interesting to see if an upper bound can also be obtained, and more generally to chart out the $b$-robust hollow star for natural classes. 

For example, we can show that the $b$-robust hollow star number for singletons on a domain of size $n$ is \textit{exactly} $(b+1)(n-1)+1$ (\Cref{sec:robust-hollows-star-singletons}). In addition, for $d$-dimensional halfspaces, one can use the same reasoning as in \Cref{example:halfspaces}, together with the \textit{Tolerance} Carath\'eodory theorem (Thereom 4.1 in \cite{montejano2011tolerance}, Theorem 2 in \cite{TUZA198984}) to obtain certificates of size $< (d+b)^{O(b)}$ for test points in the $b$-robust agreement of the training data. From \Cref{theorem:b-robust-hollow-star}, we then know that the $b$-robust hollow star number of halfspaces is at most $(d+b)^{O(b)}$ (the lower bound from \Cref{claim:b-robust-lower-bounded-by-hollow-star} is only $\Omega(bd)$.)  
One way to close the gap for halfspaces would be to show a lower bound for the Tolerance Carath\'eodory theorem. The simplest phrasing of this for $b=1$ amounts to the following purely convex-analytic question:

\begin{openquestion}
    \label{openquestion:tolerance-caratheodory}
     Can one construct a set $S$ of $\Omega(d^2)$ points in $\R^d$, and a test point $x$, such that: (1) no matter which point $x'$ is removed from $S$, $x$ is still in the convex hull of $S \setminus \{x'\}$, \\ (2) $S$ is a minimal set satisfying (1). That is, no matter which point $x_1$ is removed from $S$, there exists another point $x'$ that can be removed, such that $x$ is no longer in the convex hull of $S \setminus \{x_1, x'\}$? \\[0.05in]
     Equivalently, $S$ is a minimal set such that any hyperplane through $x$ has at least two points in $S$ on either side.
\end{openquestion}
Such a result exists for the Tolerance \textit{Helly} Theorem \citep[Theorem 3.2]{montejano2011tolerance}.

\paragraph{Update:} We were recently made aware that the answer to the question above is affirmative, as established by Bárány and Perles: see Theorems 3 and 4 in \cite{barany1990caratheodory}. In particular, for fixed $b$ and as $d \to \infty$, the $b$-robust hollow star number of $d$-dimensional halfspaces is lower-bounded by $\frac{1}{e^b(b+1)!}d^{b+1}$.

%% file: semi-distributional.tex
\section{Worst-case Distributional Bounds on Certificate Size}
\label{sec:semi-distributional}

The certificate size bounds from the previous section were from a worst-case perspective---the training data $S$ was arbitrary, and the test point $(x,y)$ was also an arbitrary point in its agreement region. A standard assumption in learning theory however is that $S$ is sampled i.i.d. from 
some marginal distribution $\mcD$ over $\mcX$, and then labeled by the unknown $h^\star$. Under this assumption, we can think about certifying a test point $x$ with the label $h^\star(x)$ as the sample $S$ gets large (we will study quantitative bounds on the size of $S$ in \Cref{sec:distributional}).

A first observation is that if the distribution $\mcD$ is discrete, and the test point $x$ belongs to the support of $\mcD$, we will eventually observe more than $2b+1$ copies of it in $S$. With a corruption budget of $b$, an adversary can potentially corrupt the labels on $b$ of these copies; however, $b+1$ copies with the true label $h^\star(x)$ still remain in $S$. We can use these copies, all of which have the true label, to certify that $x$ should be labeled as $h^\star(x)$, since any hypothesis that labels $x$ differently makes strictly more than $b$ mistakes on these copies. Furthermore, this certificate size is optimal---any certificate of smaller size is $b$-robustly realizable by every hypothesis in the class. 

We now turn our attention to the more interesting case, where either the test point $x$ is \textit{not} in the support of $\mcD$ (e.g., if there is train-test distribution shift), or $\mcD$ is a continuous distribution. Observe that in either case, it is possible to certify $x$ \textit{only if} $x$ eventually belongs to the agreement region of the (uncorrupted) sample $S$ with positive probability.
Interestingly, for distributions satisfying this property, we obtain the following sharp bound on the reliable (minimum) certificate size.

\begin{theorem}[Eventual Certification with Small Certificate]
    \label{theorem:semi-distributional-certification}
    Let $\mcH$ be a hypothesis class having hollow star number $s_0$, $h^\star \in \mcH$ be the target hypothesis, and $x$ be a test point. Let $\mcD$ be a distribution over samples labeled by $h^\star$, and suppose $\mcD, h^\star$ satisfy the property that: there exists $m$ large enough such that $\Pr_{S \sim \mcD^m}[(x,h^\star(x)) \text{ in agreement region of }S] > 0$.
    Then, for any $\delta>0$, there exists $m(\delta)$ such that a $b$-robust certificate for $(x, h^\star(x))$ of size at most $(b+1)(s_0-1)$ can be extracted from $S\sim \mathcal{D}^{m(\delta)}$ with probability $\geq 1-\delta$. Moreover, 
    there exists an instantiation of the setting where a $b$-robust certificate of smaller size is not possible.
\end{theorem}

\begin{remark}
    Recall that once we draw a large enough sample $S$ so that $(x, h^\star(x))$ is in the agreement region of $S$, \Cref{theorem:b-robust-hollow-star} guarantees the existence of a $b$-robust certificate of size at most $s_b-1$ in $S$, where $s_b$ is the $b$-robust hollow star number of $\mcH$. However, by \Cref{claim:b-robust-lower-bounded-by-hollow-star}, we know that $s_b-1 \ge (b+1)(s_0-1)$. Thus, the certificate given by \Cref{theorem:semi-distributional-certification} potentially has a better size than that implied by \Cref{theorem:b-robust-hollow-star}.
\end{remark}
\begin{proof}
    Let $m'$ be such that $\Pr_{S \sim \mcD^{m'}}[(x,h^\star(x)) \text{ in agreement region of }S]=\delta'$ for some $\delta'>0$; such an $m'$ exists by assumption that $(x,h^\star(x))$ belongs to the agreement region of a drawn sample $S$ eventually. 
    Consider drawing a sample $S$ of size $(2b+1)n m'/\delta'$, and let $S_1,S_2,\dots,S_{(2b+1)n/\delta'}$ denote independent partitions of $S$ into chunks of size $m'$. The expected number of chunks that satisfy that $(x,h^\star(x))$ is in the agreement region of the chunk is $(2b+1)n$. By a Chernoff bound, the probability that at least $2b+1$ chunks have $(x, h^\star(x))$ in their agreement region is at least $1-e^{-\Omega(bn)}$. 
    Now, with a corruption budget of size $b$, an adversary can insert a corruption in at most $b$ chunks. Even so, the remaining $\ge b+1$ chunks continue to have $(x,h^\star(x))$ in their agreement region. Without loss of generality, suppose $S_1,\dots,S_{b+1}$ are all uncorrupted. Then, from \Cref{theorem:b-robust-hollow-star} (with $b=0$), we know that each of $S_1,\dots,S_{b+1}$ will contain a certificate for $(x,h^\star(x))$ of size at most $s_0-1$; let these certificates be $C_1,\dots,C_{b+1}$. The final $b$-robust certificate is simply a concatenation of $C_1,\dots,C_{b+1}$.

    To see that this is a $b$-robust certificate, observe that it is realizable by $\mcH$ since it is uncorrupted. Now consider any $h \in \mcH$ that makes at most $b$ mistakes on the concatenation. Then, observe that $h$ makes no mistakes on some $C_i$. Since $C_i$ is a certificate for $(x, h^\star(x))$, $h$ must label $x$ as $h^\star(x)$. Setting $n=O\left(\log(1/\delta)/b\right)$ ensures that a sample $S'$ of size $m(\delta)=O\left(m'\log(1/\delta)/\delta'\right)$ suffices for failure probability at most $\delta$. 

    Finally, we argue that a $b$-robust certificate size of $(b+1)(s_0-1)$ is optimal in general for this setting. Let $\{(x_1,y_1),\dots,(x_{s_0},y_{s_0})\}$ be a hollow star for $\mcH$. Note that by the hollow star property, $\{(x_1,y_1),\dots,(x_{s_0},-y_{s_0})\}$ is realizable by some $h^\star \in \mcH$---let this $h^\star$ be the target hypothesis, and $x_{s_0}$ be the test point. Let $\mcD$ be the distribution that puts all of its mass uniformly on $x_1,\dots,x_{s_0-1}$. Note that $\mcD$ satisfies the specifications in the theorem statement. 

    Suppose that the adversary creates no corruptions. We claim that in this case, any $b$-robust certificate for $(x_{s_0}, -y_{s_0})$ must necessarily contain at least $b+1$ copies of $(x_i, y_i)$ for every $i \in \{1,\dots,s_{0}-1\}$. To see this, suppose that for some $i$, there are at most $b$ copies of $(x_i, y_i)$ included in the certificate. Since the adversary does not corrupt any samples, the rest of the samples included in the certificate comprise of (copies of) $(x_1,y_1),\dots,(x_{i-1}, y_{i-1}),(x_{i+1},y_{i+1}),\dots,(x_{s_0-1}, y_{s_0-1})$. Now, by the hollow star property, observe that there exists $h \in \mcH$, which realizes 
    $$
        \{(x_1,y_1),\dots,(x_{i-1}, y_{i-1}),(x_{i+1},y_{i+1}),\dots,(x_{s_0-1}, y_{s_0-1}), (x_{s_0},y_{s_0})\},
    $$
    and this $h$ must label $x_i$ as $-y_i$ (by non-realizability of the hollow star). Then, because the certificate only consists of at most $b$ copies of $(x_i, y_i)$, it is $b$-robustly realizable by $h$. This contradicts it being a $b$-robustly realizable certificate for $(x_{s_0},-y_{s_0})$, since $h$ labels $x_{s_0}$ as $y_{s_0}$.
\end{proof}
We end this section with a final observation. The conclusion of \Cref{theorem:semi-distributional-certification} implies that any distribution $\mcD$ which satisfies
that $(x,h^\star(x))$ eventually belongs to the agreement region of a sample drawn from $\mcD$, actually satisfies the seemingly \textit{stronger} condition: that $(x,h^\star(x))$ eventually belongs to the $b$\textit{-robust} agreement region of a sample drawn from $\mcD$. Note that this stronger condition implies the weaker condition.
Therefore, these conditions are equivalent.

%% file: distributionalarxiv.tex
\section{Distribution-dependent Bounds on Sample Size}
\label{sec:distributional}

The previous section establishes a sharp bound on the $b$-robust certificate size for a given test point that can be obtained from a sample  drawn from a distribution. This bound is \textit{distribution-independent}, in that it holds for \textit{all} distributions that merely satisfy that the test point eventually belongs to the agreement region of the drawn sample. However, it does not quantify the number of samples that need to be seen for the test point to belong to the $b$-robust agreement region. 

In this section, we derive \textit{distribution-dependent} bounds on the size of the sample $S$ that needs to be drawn for a test point $x$ to lie in the $b$-robust agreement region of $S$. Notably, observe that once $S$ contains $x$ in its $b$-robust agreement region, it is by definition a certificate for $x$. Thus, our bound for the sample size is also a bound for the $b$-robust certificate size for $x$. The sample complexity that we state is in terms of a quantity that depends on the distribution $\mcD$, target hypothesis $h^\star$ and test point $x$ we wish to certify. We term this quantity the ``Certificate Coefficient'' of $x$.

\begin{definition}[Certificate Coefficient]
    \label{def:certificate-coefficient}
    Let \(\mcH\) be a hypothesis class, \(h^\star \in \mcH\) the ground truth hypothesis, and \(\mcD\) a marginal distribution over the domain \(\mcX\). For a test point \(x \in \mcX\), denote the set of hypotheses that disagree with \(h^\star\) on \(x\) by:  
    \begin{align}
    \label{eqn:H-x}
        \mcH_x = \{h \in \mcH : h(x) \neq h^\star(x)\}.
    \end{align}
    The \textit{certificate coefficient} $\eps_x=\eps_x(\mcD, h^\star)$ of $x$ is defined as:  
    \begin{align}
        \label{eqn:certificate-coefficient}
        \eps_x = \inf_{h \in \mcH_x} \Pr_{z \sim \mcD}[h(z) \neq h^\star(z)].
    \end{align}
\end{definition}
We will refer to $\eps_x(\mcD, h^\star)$ as simply $\eps_x$ when the parameters involved can be deduced from context. Intuitively, the certificate coefficient $\eps_x$ measures how quickly the point $x$ gets embedded in the agreement region of a sample from $\mcD$. If $\eps_x$ is substantial, then $x$ is in the ``interior" of the (labeled) distribution, and we should quickly expect $x$ to be in the agreement region. If $\eps_x$ is tiny, then $x$ is at the boundary, and we have to see many points before $x$ falls in the agreement region. 

\subsection{Upper Bound on the Sample Complexity}
\label{sec:sample-complexity-upper-bound}

While our bounds above on certificate size were in terms of the $b$-robust hollow star number, our sample complexity bound is more familiar-looking and is in terms of the VC dimension of the class.

\begin{theorem}[Sample Complexity for Robust Certification]
    \label{theorem:robust_agreement}
    Let \( \mathcal{H} \) be a hypothesis class having $VC$ dimension $d < \infty$ and \( h^\star \in \mathcal{H} \) be the target hypothesis. For any marginal data distribution \( \mathcal{D} \), test point \( x \) satisfying $\eps_x > 0$, corruption budget \( b \ge 0 \) and failure probability $\delta \in (0,1)$, obtaining an i.i.d. sample $S$ (labeled by $h^\star$) of size
        $m = O\left(\frac{b + d\log(1/\eps_x) + \log(1/\delta)}{\eps_x}\right)$ 
    suffices to ensure that with probability \( 1 - \delta \), \( x \) belongs to the \( b \)-robust agreement region of $S$. 
\end{theorem}
\begin{proof}[Proof sketch]
    We sketch the proof for the case when $\mcH$ is finite. Fix some $h \in \mcH_x$. Upon drawing a sample $S$ of size $m$ i.i.d. from $\mcD$ labeled by $h^\star$, the expected number of mistakes that $h$ makes on $S$ is, by definition of the certificate coefficient, at least $\eps_x m$. Then, by a Chernoff bound, the probability that $h$ makes less than $\eps_x m/2$ mistakes is at most $\exp(-\eps_x m/8)$. Thus, setting $m = \frac{8\log(|\mcH|/\delta)}{\eps_x}$, together with a union bound, ensures that with probability at least $1-\delta$, every $h \in \mcH_x$ makes strictly more than $\eps_x m/2$ mistakes on $S$. If in addition, $m \ge 2b/\eps_x$, then this means that every $h \in \mcH_x$ makes more than $b$ mistakes on $S$. But this also means that $(x,h^\star(x))$ is in the $b$-robust agreement region of $S$. Thus, setting $m = \max\left(2b/\eps_x, \frac{8\log(|\mcH|/\delta)}{\eps_x}\right) = O\left(\frac{b + \log|\mcH|+\log(1/\delta)}{\eps_x}\right)$ is sufficient for the required guarantee. The extension to infinite classes having bounded VC dimension involves the standard double sampling + symmetrization trick, and is given in \Cref{sec:proof-robust-agreement}.
\end{proof}

\begin{remark}[Simultaneous Certification of Multiple Points]
    We observe that the proof of \Cref{theorem:robust_agreement} can be adapted to ensure that $S$ robustly certifies \textit{multiple} test points simultaneously. Concretely, for a class $\mcH$ having VC dimension $d$, target hypothesis $h^\star \in \mcH$, and marginal distribution $\mcD$, let %
    $\mcX_\eps=\{x \in \mcX:\eps_x(\mcD, h^\star) \ge \eps\}$. Observe that any $h$ that mislabels at least one point---say $x_1$---in $\mcX_\eps$ satisfies that $h \in \mcH_{x_1} \implies \Pr_{z \sim \mcD}[h(z) \neq h^\star(z)] \ge \eps$ by definition of $\eps_{x_1}$. From here, the analysis in the proof of \Cref{theorem:robust_agreement} gives us that for a sample $S$ of size $O\left(\frac{b+d\log(1/\eps)+\log(1/\delta)}{\eps}\right)$, with probability $1-\delta$, \textit{every} point in $\mcX_\eps$ belongs to the $b$-robust agreement region of $S$.\footnote{One can also analyze the probability mass of \( \mathcal{X}_\epsilon \), as it is directly related to the \textit{disagreement coefficient} \( \theta \) introduced by \cite{hanneke2007bound}. Specifically, it holds that $\Pr[x \in \mcX_\eps] = 1-\Pr[x \in \mcX \setminus \mcX_\eps] \ge 1-\theta\eps$.
    }     
\end{remark}

\subsection{Tightness of the Upper Bound}
\label{sec:sample-complexity-tightness}

Our next proposition shows that the sample complexity bound in \Cref{theorem:robust_agreement} is tight in general.

\begin{proposition}[Tightness of Sample Complexity for Robust Certification]
    For any failure probability $\delta < \frac{1}{100}$, the bound in \Cref{theorem:robust_agreement} is tight up to a constant.
\end{proposition}
\begin{proof}
    We will separately show that each individual term in the bound from \Cref{theorem:robust_agreement} is necessary. 

    \paragraph{Tightness of $\frac{b}{\eps_x}$:} Consider any class $\mcH$, target  $h^\star \in \mcH$, distribution $\mcD$ and test point $x$ that satisfies $\eps_x > 0$. Let $\mcH_x = \{h \in \mcH: h(x)\neq h^\star(x)\}$. By definition of $\eps_x$ \eqref{eqn:certificate-coefficient}, there exists $h \in \mcH_x$ that satisfies $\Pr_{z \in \mcD}[h(z)\neq h^\star(x)] \le 2\eps_x$. If we draw $m$ samples i.i.d. from $\mcD$, labeled by $h^\star$, the expected number of mistakes that $h$ makes on these samples is at most $2\eps_x m$. By a Chernoff bound, the probability that $h$ makes at least $4\eps_x m$ mistakes on the samples is at most $\exp(-2\eps_x m/3)$. Therefore, with probability at least $1-\exp(-2\eps_x m/3)$, $h$ makes strictly less than $4\eps_x m$ mistakes. If $m=\frac{b+1}{4\eps_x}$, we get that with probability at least $1-\exp(-(b+1)/6) \ge \frac{1}{100} > \delta$, $h$ makes strictly less than $b+1$ mistakes, which implies that $(x,h^\star(x))$ is not in the $b$-robust agreement region.
    
    \paragraph{Tightness of $\frac{d}{\eps_x}\log\left(\frac{1}{\eps_x}\right)$ and $\frac{1}{\eps_x}\log\left(\frac{1}{\delta}\right)$:} To show the tightness of the other two terms, we will work with the following class $\mcH$. For any given $d \ge 1$, the domain $\mcX$ of the class will be $\{x_{0},\dots,x_{k}\}$, %
    where we can pick $k$ to be any integer that satisfies $1+\sqrt{\frac{100}{d}} < \frac{1}{2}\log\left(\frac{k}{d}\right)$.\footnote{This is for the purpose of instantiating a coupon collector bound, e.g., as in Remark 20, \cite{bousquet2020proper}.} For example, $k=(100d)^{20}$ is a valid choice. The hypothesis class $\mcH$ comprises of $\binom{k}{d}+1$ hypotheses. For every $S \subset \{x_1,\dots,x_k\}, |S|=d$, there is a hypothesis $h_S \in \mcH$ that labels exactly the set $S$ as 1, and the rest of the domain as $-1$. Additionally, there is a special hypothesis $h^\star \in \mcH$ which satisfies that $h^\star(x_0)=1, h^\star(x_i)=-1, \forall i \in [k]$. That is, $h^\star$ only labels $x_0$ as 1. Moreover, $h^\star$ is the only hypothesis in $\mcH$ that labels $x_0$ as $1$; all the other $\binom{k}{d}$ hypotheses label $x_0$ as $-1$. It can be readily verified that the VC dimension of $\mcH$ is $d$. In what follows, we will set $b=0$ for convenience. 
    
    We will set the target hypothesis to be $h^\star$, the data distribution $\mcD$ to be uniform on $\{x_1,\dots,x_k\}$, and the test point $x$ to be $x_0$. Crucially, $\mcD$ has zero mass on $x_0$.
    \footnote{We can also allow a mass of $o(\eps_x/(d\log (1/\eps_x)))$ on $x_0$---this ensures $x_0$ is not seen in the sample with good chance.}
    Recall that $h^\star(x_0)=1$ and $h^\star(x_i)=-1, \forall i \in [k]$. Since every other hypothesis in $\mcH$ labels $x_0$ as $-1$, and also labels some $d$ points in $\{x_1,\dots,x_k\}$ as 1, we have that $\eps_x=\frac{d}{k}$. Now suppose that we obtain a sample of size $m=\frac{d}{2\eps_x}\log\left(\frac{1}{\eps_x}\right)=\frac{k}{2}\log\left(\frac{k}{d}\right)$. Then, by a coupon collector argument (see Lemma 19 in \cite{bousquet2020proper}), we have that with probability at least $\frac{1}{100} > \delta$, the number of distinct elements from $\{x_1,\dots,x_k\}$ that we see in the sample are at most $k-d$. But this means that there is some subset $S \subset \{x_1,\dots,x_k\}$ of $d$ elements whose labels we have not seen. Thus, $h_S$ is consistent with the labeled data seen so far, rendering $(x,h^\star(x))$ to not be in the agreement region of the sample.

    Finally, note also that upon drawing $m$ samples, the probability that we do not see any point in $S=\{x_1,\dots,x_d\}$ is $\left(1-\frac{d}{k}\right)^{m} \ge \exp(-2dm/k)$, which is larger than $\delta$ when $m < \frac{k}{2d}\log\left(\frac{1}{\delta}\right)=\frac{1}{2\eps_x}\log\left(\frac{1}{\delta}\right)$. In the absence of any point in $S$, 
    $(x,h^\star(x))$ will not be in the agreement region.

    Summarily, we have shown that any (generic) upper bound on the sample size $m$, which guarantees that a test point $x$ belongs to the $b$-robust agreement region of the sample necessarily satisfies
    \begin{align*}
        m \ge \max \left\{\frac{b+1}{4\eps_x}, \frac{d}{2\eps_x}\log\left(\frac{1}{\eps_x}\right), \frac{1}{2\eps_x}\log\left(\frac{1}{\delta}\right)\right\} = \Omega\left(\frac{b + d\log(1/\eps_x) + \log(1/\delta)}{\eps_x}\right),
    \end{align*}
    which completes the proof of the proposition and establishes the tightness of \Cref{theorem:robust_agreement}.
\end{proof}

\section{Shorter Certificates by Reweighting}
\label{sec:reweighting}

\Cref{theorem:robust_agreement} gives a quantitative bound on the number of samples required to ensure that a test point $x$ is in the $b$-robust agreement region (and hence, also certifies $x$); however, the bound scales inversely with the certificate coefficient $\eps_x$. While our bound from \Cref{theorem:b-robust-hollow-star} guarantees the existence of a potentially smaller certificate of size equal to the $b$-robust hollow star number, we might not have tight estimates of this quantity, as well as a constructive procedure to extract the shorter certificate. This motivates us to consider cleaner algorithmic primitives that might lead to short certificates. 

\begin{example}[Rejection Sampling]
    \label{example:ball-reweighting}
    Consider the case when $\mcH$ is the class of halfspaces in $\R^d$ (not necessarily passing through the origin), and $\mcD$ is uniform on the unit ball $\{x:\|x\|_2 \le 1\}$. Suppose that the target halfspace $h^\star$ labels every point in the ball positively (i.e., it does not cut through the ball), and that the test point $x$ is at $(1/2,0,0,\dots,0$). We can verify that $\eps_x \le 2^{-\Omega(d)}$ (as realized by the halfspace $x_1 < 1/2$; see Theorem 2.7 in \cite{blum2020foundations} for the volume of the ball excluded by this halfspace), and so, the upper bound in \Cref{theorem:robust_agreement} would suggest that we draw $b \cdot 2^{\Omega(d)}$ samples to ensure that $x$ lies in the $b$-robust agreement region of the samples with constant probability. We might extract a shorter certificate from this sample using the worst-case bound from \Cref{theorem:b-robust-hollow-star}, but even this scales as the $b$-robust hollow star number of $\mcH$, for which the best upper bound is $(d+b)^{O(b)}$. Consider the following alternate recipe. 
    Suppose we discard any samples from $\mcD$ that are not contained in the ball of radius $1/2$ centered at $x$. This induces the uniform distribution $\mcD_w$ on the smaller ball, and in expectation, we obtain one sample from $\mcD_w$ for every $2^d$ samples from $\mcD$. Crucially, notice that $\eps_x(\mcD_w,h^\star)=1/2$. By \Cref{theorem:robust_agreement}, $O(b+d)$ samples from $\mcD_w$ are sufficient to ensure that $x$ lies in the $b$-robust agreement region of the samples with constant probability. Because of our rejection sampling, this really requires obtaining $(b+d)\cdot2^{O(d)}$ samples from $\mcD$. The final result is that we have a $b$-robust certificate of size only $O(b+d)$ for $x$. This required us to draw only a polynomially larger sample than we would have if we directly applied \Cref{theorem:robust_agreement}.
\end{example}
In the example above, we manipulated the distribution of the data to boost up the certificate coefficient. This allowed us to obtain a better sample size bound, resulting also in a better \textit{certificate} size! In the process, we maintained that the number of samples required is only a polynomial factor larger in the (original) problem parameters. This procedure motivates the following definition:
\begin{definition}[Optimal Reweighted Coefficient]
    \label{def:optimal-reweighted-coefficient}
    Let \(\mcH\) be a hypothesis class having VC dimension $d$, \(h^\star \in \mcH\) the target hypothesis, $b \ge 0$ a corruption budget, and \(\mcD\) a distribution over \(\mcX\). Let $x$ be a test point having $\eps_x(\mcD, h^\star) > 0$. We say that $w: \mcX \to [0,1]$ is a valid reweighting of distribution $\mcD$, resulting in the distribution $\mcD_w$ where, $\mcD_w(z) \propto w(z) \cdot \mcD(z)$, if it satisfies:
    \begin{align}
        \label{eqn:polynimal-reweighting-condition}
        Z = \int_z  w(z) \mathcal{D}(z) \, dz \ge \frac{1}{\poly(b,d, 1/\eps_x)}.
    \end{align}
    Then, we define $\eps^\star_x$ as the supremum of certificate coefficients over valid reweightings of $\mcD$:
    \begin{align}
        \label{eqn:eps_star}
        \eps^{\star}_x = \sup_{\substack{w : \mathcal{X} \to [0,1] \\ w \text{ valid }}} \eps_x(\mcD_w,h^\star) = \sup_{\substack{w : \mathcal{X} \to [0,1] \\ w \text{ valid }}} \left\{ \inf_{h \in \mcH_x} \Pr_{z \sim \mathcal{D}_w}[h(z) \neq h^\star(z)] \right\}.
    \end{align}
    Here, $\mcH_x$ is as defined in \eqref{eqn:H-x}.
\end{definition}

For any valid reweighting $w$ of $\mcD$, one can draw a sample from $\mcD_w$ as follows: until a sample gets accepted, draw \( z \sim \mathcal{D} \), and accept $z$ with probability  \( w(z)\). Then, the rejection sampling procedure sketched out in \Cref{example:ball-reweighting} gives the following theorem, whose proof is in \Cref{sec:proof-reweighted-certificate}.

\begin{theorem}[Certificates by Reweighting]
    \label{theorem:reweighted-certificate}
    Let \( \mathcal{H} \) be a hypothesis class having $VC$ dimension $d < \infty$ and \( h^\star \in \mathcal{H} \) be the target hypothesis. For any marginal data distribution \( \mathcal{D} \), test point \( x \) satisfying $\eps_x=\eps_x(\mcD, h^\star) > 0$, corruption budget \( b \ge 0 \) and failure probability $\delta \in (0,1)$, with probability $1-\delta$, we can obtain a certificate for $(x, h^\star(x))$ of size $O\left(\frac{b + d\log(1/\eps_x^{\star}) + \log(1/\delta)}{\eps_x^{\star}} \right)$ by obtaining an i.i.d. sample $S$ (labeled by $h^\star$) of size $\poly(b,d,1/\eps_x,1/\delta)$.
\end{theorem}

\begin{remark}
    Given that we were able to boost the certificate coefficient all the way up to a constant in \Cref{example:ball-reweighting}, one might ask: is this always possible? If it were, the rejection sampling procedure from \Cref{theorem:reweighted-certificate} would always guarantee a certificate of size just $O(b+d)$ once the sample gets sufficiently large. Unfortunately, the lower bound from \Cref{theorem:semi-distributional-certification} prevents this from happening; there are instances where $\eps^\star_x \lesssim \frac{1}{s_0}+\frac{d}{bs_0}$ (ignoring log factors), where $s_0$ is the hollow star number.
\end{remark}
\Cref{example:ball-reweighting} and the guarantee of \Cref{theorem:reweighted-certificate} suggest that reweighting/rejection sampling can be fruitful in obtaining short certificates. However, as stated, the computation of $\eps^\star_x$ requires knowledge of the distribution $\mcD$---this is indeed a very strong assumption, leading to our next open question.

\begin{openquestion}
    \label{openquestion:reweighting}
     Can we construct general-purpose reweighting schemes (possibly employing some form of iterated multiplicative weights like AdaBoost) that do not require knowing $\mcD$, but implicitly converge to a weighting $w$ for which $\eps_x(\mcD_w, h^\star) \approx \eps^\star_x$?
\end{openquestion}

\section{Discussion}
In this work, we introduced the notion of short certificates: subsets of the training data that provably determine the correct label of a test point \( x \) even under up to \( b \) corruptions. To characterize their worst-case size, we proposed the robust hollow star number, a generalization of the hollow star number of \cite{bousquet2020proper}. We then studied worst-case distributional bounds and introduced the certificate coefficient \( \eps_x \), which captures the distribution-dependent difficulty of certifying a given point. For this setting, we established bounds on the sample size required to certify \( x \), as a function of \( \eps_x \), the corruption budget \( b \), and the VC-dimension \( d \) of the class \( \cH \). We further showed that reweighted variants of \( \eps_x \), can lead to improved bounds from polynomial-sized samples.
Our framework also naturally subsumes the agnostic learning setting as a special case of adversarial model, under the assumption that the true label of \( x \) is given by the hypothesis closest to the target function.  Following \cite{balcan2022robustly}, a \( b \)-robust certificate becomes a \( b \)-agnostic certificate when the corruption budget reflects the error rate of the best-in-class hypothesis: a set \( S \) certifies \( (x, y) \) if every hypothesis predicting \( 1 - y \) on \( x \) incurs more than \( b \) errors on \( S \). We believe Open Question \ref{openquestion:reweighting} poses an exciting direction for future research.
\vspace{0.5cm}
\subsection*{Acknowledgments}
This work was supported in part by the National Science Foundation under grants CCF-2212968,
ECCS-2216899, and ECCS-2217023, by the Simons Foundation under the Simons Collaboration on
the Theory of Algorithmic Fairness, and by the Office of Naval Research MURI Grant N000142412742. We thank Shay Moran for bringing to our notice the work by Bárány and Perles \cite{barany1990caratheodory}, which resolves Open Question \ref{openquestion:tolerance-caratheodory}.

%% file: appendix.tex
\section{Supplementary Proofs}
\label{sec:proofs}

\subsection{Proof of \Cref{claim:b-robust-lower-bounded-by-hollow-star}}
\label{sec:proof-b-robust-lower-bounded-by-hollow-star}

Let $\{(x_1,y_1),\ldots,(x_{s_0},y_{s_0})\}$ be a hollow star for $\mcH$, and consider $S=\{(x_1,y_1),\ldots,(x_{s_0-1},y_{s_0-1})\}$ which contains all but the last element of the hollow star. We will show that the sequence $T=\{S_1,\ldots,S_{b+1},(x_{s_0},y_{s_0})\}$, where each $S_i$ is a copy of $S$, is a $b$-robust hollow star for $\mcH$. Since $T$ has size $(b+1)(s_0-1)+1$, this will prove the claim.

Consider assigning a weight $1$ to every element in $S_1,\ldots,S_{b+1}$, and the weight $b+1$ to the last element $(x_{s_0}, y_{s_0})$. Then, the weighted sequence is not $b$-robustly realizable by $\mcH$. To see this, consider any $h \in \mcH$. If $h(x_{s_0}) \neq y_{s_0}$, then the weighted error on $x_{s_0}$ is already $b+1$. If $h(x_{s_0}) = y_{s_0}$, then there must exist some $k \in \{1,2,\dots,s_{0}-1\}$ such that $h(x_k) \neq y_k$. Otherwise, the sequence $\{(x_1,y_1),\ldots,(x_{s_0},y_{s_0})\}$ would be realizable by $h$, which violates the fact that this sequence is a hollow star. Since there exist $b+1$ copies of $(x_k,y_k)$ in $T$, the total weighted error of $h$ is again at least $b+1$.

We now argue that removing any (weighted) element of $T$ makes it $b$-robustly realizable by $\mcH$. If we remove $(x_{s_0}, y_{s_0})$, we are simply left with $b+1$ copies of $S$. But $S$ is realizable by $\mcH$, since it excludes the last element of the hollow star. On the other hand, suppose we remove $(x_k, y_k)$ from $S_i$ in $T$, for some $k \in \{1,2,\dots,s_0-1\}, i \in \{1,2,\dots,b+1\}$. Then, observe (by the hollow star property again) that there must exist some $h \in \mcH$ which realizes $(S_i \setminus \{(x_k, y_k)\}) \cup \{(x_{s_0},y_{s_0})\}$. This $h$ therefore makes a mistake only on the remaining $b$ copies of $(x_k,y_k)$, each of which has weight $1$. We conclude the proof. \hfill$\blacksquare$

\subsection{$b$-robust Hollow Star Number for Singletons}
\label{sec:robust-hollows-star-singletons}

\begin{claim}[Singletons $b$-robust hollow star]
     Consider the class $\mcH$ of singletons on a domain of size $n$, i.e., $\mcX = [n]$, $\mcH = \{x \mapsto (-1)^{1-\Ind[x=i]}: i \in [n]\}$. The $b$-robust hollow star number of $\mcH$ is exactly equal to $(b+1)(n-1)+1$.
\end{claim}
\begin{proof}
    Let $s_b$ be the $b$-robust hollow star number of $\mcH$. The (0-robust) hollow star number of $\mcH$ is equal to $n$, and hence by \Cref{claim:b-robust-lower-bounded-by-hollow-star}, we know that $s_b \ge (b+1)(n-1)+1$. We will now show that $s_b \le (b+1)(n-1)+1$, which will prove the claim.

    Suppose that $T=(x_1,y_1,b+1),(x_2,y_2,1),\dots,(x_{s_b},y_{s_b},1)$ is a $b$-robust hollow star for $\mcH$, where as per \Cref{def:robust-hollow-star}, the weight on $(x_1,y_1)$ is $b+1$, and the weight on the rest of the points is $1$. Note that no matter what $y_1$ is, there exists some $h_{i'} \in \mcH$ that labels $x_1$ as $-y_1$. 
    
    Now, consider the sequence $T \setminus \{(x_{s_b},y_{s_b},1)\}$.
    There should exist $h_{i_1} \in \mcH$ that $b$-robustly realizes this sequence. In particular, $h_{i_1}$ must label $x_1$ as $y_1$ (hence, $h_{i_1}$ cannot be $h_{i'}$). Furthermore, it must be the case that $h_{i_1}$ labels $x_{s_b}$ as $-y_{s_b}$; otherwise, $h_{i_1}$ would $b$-robustly realize all of $T$. Thus, $h_{i_1}$ is a hypothesis in $\mcH$ which: (1) is not $h_{i'}$, (2) makes a mistake on $(x_{s_b},y_{s_b})$, and (3) makes at most $b$ mistakes on $(x_{2},y_2),\dots,(x_{s_b-1},y_{s_b-1})$.

    Repeating this argument with $T \setminus \{(x_{s_b-1},y_{s_b-1},1)\},T \setminus \{(x_{s_b-2},y_{s_b-2},1)\},\dots,T \setminus \{(x_{2},y_{2},1)\}$, we will have obtained $h_{i_1},h_{i_2},\dots,h_{i_{s_{b}-1}}$, such that for every $j \in [s_{b}-1]$, it holds that (1) $h_{i_j}$ is not $h_{i'}$, (2) $h_{i_j}$ makes a mistake on $(x_{s_{b}-j+1}, y_{s_{b}-j+1})$, and (3) $h_{i_j}$ makes at most $b$ mistakes on 
    $$
        \{(x_{2},y_2),\dots,(x_{s_b},y_{s_b})\} \setminus \{(x_{s_{b}-j+1}, y_{s_b}-j+1)\}.
    $$
    Thus, we can think of the $n$ hypotheses in $\mcH$ as $n$ bins, and the $s_b-1$ points $(x_2,y_2),\dots,(x_{s_b},y_{s_b})$ as $s_b-1$ balls. From conditions (1) and (2), we get that each of the $s_b-1$ balls is assigned one of $n-1$ bins from the set $\mcH \setminus \{h_{i'}\}$. From condition (3), each of the $n-1$ bins is assigned no more than $b+1$ balls. Thus, it must hold that
    $(b+1)(n-1) \ge s_b-1$, which completes the proof.
\end{proof}

\subsection{Proof of \Cref{theorem:robust_agreement}}
\label{sec:proof-robust-agreement}

Using the standard double sampling+symmetrization argument, we will generalize the bound sketched in the main body above for finite hypothesis classes to arbitrary (potentially infinite) hypothesis classes $\mcH$ having VC dimension $d$.

Again, let $h^\star \in \mcH$ be the target hypothesis, $\mcD$ the marginal distribution on the data, and $x$ the test point. Let $\mcH_x$ and $\eps_x$ be defined as in \eqref{eqn:H-x} and \eqref{eqn:certificate-coefficient}.

Suppose we draw a sample $S \sim \mcD^m$ of size $m$, labeled by $h^\star$. Let $A$ be the event: there exists $h \in \mcH_x$ which makes at most $\eps_x m/4$ mistakes on $S$. We want to show that $\Pr_{S \sim \mcD^m}[A] \le \delta$.

Towards this, consider (purely for the sake of analysis) drawing an additional sample $S' \sim \mcD^m$ independently of $S$. Let $B$ be the event: there exists $h \in \mcH_x$ which makes at most $\eps_x m/4$ mistakes on $S$, and \textit{at least} $\eps_x m/2$ mistakes on $S'$. We claim that $\Pr[A] \le 2\Pr[B]$. To see this, observe that
\begin{align*}
    \Pr[B] &\ge \Pr[A \text{ and } B] \ge \Pr[A] \cdot \Pr[B|A].
\end{align*}
Now, to bound $\Pr[B|A]$, let $h \in \mcH_x$ be the hypothesis which makes at most $\eps_x m/4$ mistakes on $S$ (this $h$ exists since we condition on $A$). Since $h \in \mcH_x$, its expected mistakes on $S'$ are at least $\eps_x m$. So, since $S'$ is independent of $S$, the probability that this $h$ makes at least $\eps_xm/2$ mistakes on $S'$ is, by a Chernoff bound, at least $1/2$, provided $m \ge 6/\eps_x$.

So, we continue to argue that $\Pr[B] \le \delta/2$. First, instead of drawing $S$ and $S'$ independently from $\mcD^m$, we simply draw $S'' \sim \mcD^{2m}$, and think of the first $m$ points as $S$ and the second $m$ points as $S'$. This is a distributionally identical way of obtaining $S$ and $S'$. Now, we make an observation. Once we have drawn $S''$ (and therefore, $S$ and $S'$), the event $B$ only depends on the projection of $\mcH_x$ on $S''$. That is,
\begin{align*}
    &\Pr_{S''}[B] \\
    &= \Pr_{S''}[\exists h \in \mcH_x|_{S''} \text{ s.t. $h$ makes at most $\eps_x m/4$ mistakes on $S$, but at least $\eps_x m/2$ mistakes on $S'$}].
\end{align*}
In the form above, once $S''$ is drawn, the event in the parentheses is determined. So, we introduce some additional randomness. Suppose
$$
S'' = (x_1, h^\star(x_1)), \ldots, (x_m, h^\star(x_m)), (x_{m+1}, h^\star(x_{m+1}),\ldots, (x_{2m}, h^\star(x_{2m})).
$$
We now flip $m$ independent fair coins $p_1,\dots,p_m$. If $p_i$ lands heads, we swap $(x_i, h^\star(x_i))$ and $(x_{m+i}, h^\star(x_{m+i}))$, otherwise we let this pair be as is. After doing this for all the $m$ coins, we let the first $m$ points be $S$ and the second $m$ points be $S'$. Again, this is a distributionally identical way of obtaining $S$ and $S'$. So, we have that
\begin{align*}
&\Pr_{S''}[\exists h \in \mcH_x|_{S''} \text{ s.t. $h$ makes at most $\eps_x m/4$ mistakes on $S$, but at least $\eps_x m/2$ mistakes on $S'$}] = \\
&= \Ex_{S''}\left[\Pr_{p_1,\dots,p_m}[\exists h \in \mcH_x|_{S''} \text{ s.t. $h$ makes at most $\eps_x m/4$ mistakes on $S$, but at least $\eps_x m/2$ mistakes on $S'$}]\right] \\
&\le \Ex_{S''}\left[ \sum_{h \in \mcH_x|_{S''}}\Pr_{p_1,\dots,p_m}[\text{$h$ makes at most $\eps_x m/4$ mistakes on $S$, but at least $\eps_x m/2$ mistakes on $S'$}] \right].
\end{align*}
Now, for any fixed $S''$ and $h \in \mcH_x|_{S''}$, we will bound the probability of the event in the parentheses, where the randomness is \textit{only} over the coin flips.

First, we can assume that $S''$ satisfies: for at most $\eps_x/4$ pairs $(i, m+i)$, $h$ makes a mistake at both $x_i$ and $x_{m+i}$, and for at least $\eps_x m/2$ pairs $(i, m+i)$, $h$ makes a mistake on at least one of $x_i$ and $x_{m+i}$. If $S''$ does not satisfy the first condition, then $h$ makes strictly more than $\eps_x m/4$ mistakes on $S$. Similarly, if $S''$ does not satisfy the second condition, then $h$ makes strictly less than $\eps_x m/2$ mistakes on $S'$. Thus, if $S''$ does not satisfy either of these conditions, the required event cannot happen.

However, these two conditions together mean that there are at least $\eps_x m/4$ pairs $(i, m+i)$, such that $h$ makes a mistake on \textit{exactly} one of $x_i$ or $x_{m+i}$. And then, for the event to occur, namely for $h$ to make at least $\eps_x m/2$ mistakes on $S'$, it must be the case that the coin flips at \textit{all} these indices direct the mistake towards $S'$ (namely the second half of $S''$). The probability of this happening is at most $2^{-\eps_x m / 4}$, which is at most $\delta/2\tau_{\mcH}(2m)$, provided $m \ge \frac{4\log(2\tau_{\mcH}(2m)/\delta)}{\eps_x}$. Here, $\tau_{\mcH}$ denotes the \textit{growth function} of $\mcH$; namely, $\tau_{\mcH}(2m)=\max_{T \in \mcX^{2m}}|\mcH|_T|$.

Thus, we can conclude saying that so long as $m \ge \frac{10\log(2\tau_{\mcH}(2m)/\delta)}{\eps_x}$,
\begin{align*}
    \Pr[A] \le 2\Pr[B] \le 2\Ex_{S''}\left[\sum_{h \in \mcH_x|_{S''}}\delta/2\tau_{\mcH}(2m)\right] \le 2\tau_{\mcH}(2m) \cdot \delta/2\tau_{\mcH}(2m) = \delta,
\end{align*}
as required. That is, for $m$ satisfying this bound, every $h \in \mcH_x$ makes strictly more than $\eps_x m/4$ mistakes on $S$. By ensuring the strong condition of $m  \ge \max\left\{\frac{4b}{\eps_x},\frac{10\log(2\tau_{\mcH}(2m)/\delta)}{\eps_x}\right\}$, we will have that with probability at least $1-\delta$, every $h \in \mcH_x$ makes strictly more than $b$ mistakes on $S$, meaning also that $(x,h^\star(x))$ is in the $b$-robust agreement region of $S$.

Finally, since the VC dimension of $\mcH$ is $d$, the Sauer-Shelah-Perles lemma ensures that $\tau_{\mcH}(2m) \le (2me/d)^d$. Plugging this in our bound gives us that $m = O\left(\frac{b+d\log(1/\eps_x)+\log(1/\delta)}{\eps_x}\right)$ is sufficient to ensure that with probability at least $1-\delta$, $(x,h^\star(x))$ is in the $b$-robust agreement region of $S \sim \mcD^m$, finishing the proof. \hfill$\blacksquare$

\subsection{Proof of \Cref{theorem:reweighted-certificate}}
\label{sec:proof-reweighted-certificate}

We recall that $\eps_x = \eps_x(\mcD, h^\star)$. By definition of $\eps^\star_x$ \eqref{eqn:eps_star}, there exists a valid reweighting $w:\mcX\to[0,1]$ and corresponding $\mcD_w$ which satisfies that $\eps^w_x=\eps^w_x(\mcD_w,h^\star) \ge \eps^\star_x/2$. Consider the rejection sampling procedure for drawing a sample from $\mcD_w$: until a sample gets accepted, draw \( z \sim \mathcal{D} \), and accept $z$ with probability  \( w(z)\).

We will show that drawing $\poly(b,d,1/\eps_x,1/\delta)$ samples from $\mcD$ is sufficient to ensure that with probability $1-\delta/2$, the number of samples accepted is at least $6 \left(\frac{b+d\log(1/\eps^w_x)+\log(2/\delta)}{\eps^w_x}\right)$. Conditioned on this event, note that all the accepted samples are distributed according to $\mcD_w$. Denote the first $m_w=6 \left(\frac{b+d\log(1/\eps^w_x)+\log(2/\delta)}{\eps^w_x}\right)$ of the accepted samples as $S'$. Then, \Cref{theorem:robust_agreement} immediately gives us that with probability at least $1-\delta/2$, $(x,h^\star(x))$ belongs to the $b$-robust agreement region of $S'$. A union bound over both the $\delta/2$ failure probability events ensures that $S'$ thus obtained is the required certificate from the theorem statement, with probability $1-\delta$. The size of $S'$ is $m_w = 6 \left(\frac{b+d\log(1/\eps^w_x)+\log(2/\delta)}{\eps^w_x}\right) \le 12 \left(\frac{b+d\log(2/\eps^\star_x)+\log(2/\delta)}{\eps^\star_x}\right)$ as required.

So, we continue to argue that if we draw $\poly(b,d,1/\eps_x,1/\delta)$ samples from $\mcD$, then at least $6 \left(\frac{b+d\log(1/\eps^w_x)+\log(2/\delta)}{\eps^w_x}\right)$ samples are accepted. Towards this, note that by definition of the sampling process, a sample drawn from $\mcD$ gets accepted with probability $p=\int_{z}w(z)\mcD(z)dz$. Thus, if we draw $m$ samples i.i.d. from $\mcD$, the expected number of accepted samples is $pm$. By a Chernoff bound, the probability that less than $pm/2$ samples get accepted is at most $\exp(-pm/12)$, which is at most $\delta/2$ provided $m \ge \frac{12\log(2/\delta)}{p}$. Thus, for $m$ satisfying this condition, we have that with probability at least $1-\delta/2$, at least $pm/2$ samples are accepted.

Consider setting $m = 12 \cdot \poly(b,d,1/\eps_x) \left(\frac{b+d\log(1/\eps^w_x)+\log(2/\delta)}{\eps^w_x}\right)$. First, notice that this satisfies the condition $m \ge \frac{12\log(2/\delta)}{p}$, since
\begin{align*}
    12 \cdot \poly(b,d,1/\eps_x) \left(\frac{b+d\log(1/\eps^w_x)+\log(2/\delta)}{\eps^w_x}\right) &\ge \frac{12}{p} \cdot \left(\frac{b+d\log(1/\eps^w_x)+\log(2/\delta)}{\eps^w_x}\right) \\
    &\ge \frac{12\log(2/\delta)}{p},
\end{align*}
where in the first inequality, we used the condition that $w$ is a valid reweighting \eqref{eqn:polynimal-reweighting-condition}. Therefore, for this setting of $m$, with probability at least $1-\delta/2$, the number of accepted samples is at least
\begin{align*}
    pm/2 &\ge \frac{6}{\poly(b,d,1/\eps_x)} \cdot \poly(b,d,1/\eps_x) \left(\frac{b+d\log(1/\eps^w_x)+\log(2/\delta)}{\eps^w_x}\right) \\
    &\ge 6\left(\frac{b+d\log(1/\eps^w_x)+\log(2/\delta)}{\eps^w_x}\right),
\end{align*}
as required. In the first inequality above, we again used that $w$ is a valid reweighting. To conclude the proof, we note that since $2\eps^w_x \ge \eps^\star_x \ge \eps_x$,
\begin{align*}
    m &= 12 \cdot \poly(b,d,1/\eps_x) \left(\frac{b+d\log(1/\eps^w_x)+\log(2/\delta)}{\eps^w_x}\right)  \\
    &\le 24 \cdot \poly(b,d,1/\eps_x) \left(\frac{b+d\log(2/\eps^\star_x)+\log(2/\delta)}{\eps^\star_x}\right) = \poly(b,d,1/\eps_x,1/\delta).
\end{align*}
\hfill$\blacksquare$

%% file: main.bbl
\newcommand{\etalchar}[1]{$^{#1}$}
\begin{thebibliography}{MKSKRJ15}

\bibitem[AAES{\etalchar{+}}23]{ALI2023101805}
Sajid Ali, Tamer Abuhmed, Shaker El-Sappagh, Khan Muhammad, Jose~M.
  Alonso-Moral, Roberto Confalonieri, Riccardo Guidotti, Javier {Del Ser},
  Natalia Díaz-Rodríguez, and Francisco Herrera.
\newblock Explainable artificial intelligence (xai): What we know and what is
  left to attain trustworthy artificial intelligence.
\newblock {\em Information Fusion}, 99:101805, 2023.

\bibitem[ABL17]{awasthi2017power}
Pranjal Awasthi, Maria~Florina Balcan, and Philip~M Long.
\newblock The power of localization for efficiently learning linear separators
  with noise.
\newblock {\em Journal of the ACM (JACM)}, 63(6):1--27, 2017.

\bibitem[BBHS22]{balcan2022robustly}
Maria-Florina Balcan, Avrim Blum, Steve Hanneke, and Dravyansh Sharma.
\newblock Robustly-reliable learners under poisoning attacks.
\newblock In {\em Conference on Learning Theory}, pages 4498--4534. PMLR, 2022.

\bibitem[BH20]{bereg2020computing}
Sergey Bereg and Mohammadreza Haghpanah.
\newblock Computing the caratheodory number of a point.
\newblock In {\em Canadian Conference on Computational Geometry}, 2020.

\bibitem[BH21]{balcan2021noise}
Maria-Florina Balcan and Nika Haghtalab.
\newblock Noise in classification.
\newblock {\em Beyond the Worst-Case Analysis of Algorithms}, page 361, 2021.

\bibitem[BHK20]{blum2020foundations}
Avrim Blum, John Hopcroft, and Ravindran Kannan.
\newblock {\em Foundations of data science}.
\newblock Cambridge University Press, 2020.

\bibitem[BHMZ20]{bousquet2020proper}
Olivier Bousquet, Steve Hanneke, Shay Moran, and Nikita Zhivotovskiy.
\newblock Proper learning, helly number, and an optimal svm bound.
\newblock In {\em Conference on Learning Theory}, pages 582--609. PMLR, 2020.

\bibitem[BHPS23]{balcan2023reliablelearningchallengingenvironments}
Maria-Florina~F Balcan, Steve Hanneke, Rattana Pukdee, and Dravyansh Sharma.
\newblock Reliable learning in challenging environments.
\newblock {\em Advances in Neural Information Processing Systems},
  36:48035--48050, 2023.

\bibitem[BKLT22]{blanc2022query}
Guy Blanc, Caleb Koch, Jane Lange, and Li-Yang Tan.
\newblock The query complexity of certification.
\newblock In {\em Proceedings of the 54th Annual ACM SIGACT Symposium on Theory
  of Computing}, pages 623--636, 2022.

\bibitem[BLT21]{blanc2021provablyefficientsuccinctprecise}
Guy Blanc, Jane Lange, and Li-Yang Tan.
\newblock Provably efficient, succinct, and precise explanations.
\newblock {\em Advances in Neural Information Processing Systems},
  34:6129--6141, 2021.

\bibitem[BNS{\etalchar{+}}06]{10.1145/1128817.1128824}
Marco Barreno, Blaine Nelson, Russell Sears, Anthony~D. Joseph, and J.~D.
  Tygar.
\newblock Can machine learning be secure?
\newblock In {\em Proceedings of the 2006 ACM Symposium on Information,
  Computer and Communications Security}, ASIACCS '06, page 16–25, New York,
  NY, USA, 2006. Association for Computing Machinery.

\bibitem[BP90]{barany1990caratheodory}
Imre Barany and M~Perles.
\newblock The caratheodory number for the k-core.
\newblock {\em Combinatorica}, 10(2):185--194, 1990.

\bibitem[BS96]{Breiman1996BORNAT}
L.~Breiman and Nong Shang.
\newblock Born again trees.
\newblock 1996.

\bibitem[BS24]{blum2024regularized}
Avrim Blum and Donya Saless.
\newblock {Regularized Robustly Reliable Learners and Instance Targeted
  Attacks}.
\newblock {\em {arXiv preprint arXiv:2410.10572}}, 2024.

\bibitem[CLL{\etalchar{+}}17]{chen2017targetedbackdoorattacksdeep}
Xinyun Chen, Chang Liu, Bo~Li, Kimberly Lu, and Dawn~Xiaodong Song.
\newblock Targeted backdoor attacks on deep learning systems using data
  poisoning.
\newblock {\em ArXiv}, abs/1712.05526, 2017.

\bibitem[CS95]{NIPS1995_45f31d16}
Mark Craven and Jude Shavlik.
\newblock Extracting tree-structured representations of trained networks.
\newblock In D.~Touretzky, M.C. Mozer, and M.~Hasselmo, editors, {\em Advances
  in Neural Information Processing Systems}, volume~8. MIT Press, 1995.

\bibitem[DDN{\etalchar{+}}23]{10.1145/3561048}
Rudresh Dwivedi, Devam Dave, Het Naik, Smiti Singhal, Rana Omer, Pankesh Patel,
  Bin Qian, Zhenyu Wen, Tejal Shah, Graham Morgan, and Rajiv Ranjan.
\newblock Explainable ai (xai): Core ideas, techniques, and solutions.
\newblock {\em ACM Comput. Surv.}, 55(9), January 2023.

\bibitem[DVK17]{doshivelez2017rigorousscienceinterpretablemachine}
Finale Doshi-Velez and Been Kim.
\newblock Towards a rigorous science of interpretable machine learning.
\newblock {\em arXiv preprint arXiv:1702.08608}, 2017.

\bibitem[FLA23]{ferry2023learninghybridinterpretablemodels}
Julien Ferry, Gabriel Laberge, and U.~A{\"i}vodji.
\newblock Learning hybrid interpretable models: Theory, taxonomy, and methods.
\newblock {\em ArXiv}, abs/2303.04437, 2023.

\bibitem[FLMM24]{pmlr-v237-frost24a}
Nave Frost, Zachary Lipton, Yishay Mansour, and Michal Moshkovitz.
\newblock Partially interpretable models with guarantees on coverage and
  accuracy.
\newblock In Claire Vernade and Daniel Hsu, editors, {\em Proceedings of The
  35th International Conference on Algorithmic Learning Theory}, volume 237 of
  {\em Proceedings of Machine Learning Research}, pages 590--613. PMLR, 25--28
  Feb 2024.

\bibitem[GFH{\etalchar{+}}21]{geiping2021witchesbrewindustrialscale}
Jonas Geiping, Liam~H Fowl, W.~Ronny Huang, Wojciech Czaja, Gavin Taylor,
  Michael Moeller, and Tom Goldstein.
\newblock Witches' brew: Industrial scale data poisoning via gradient matching.
\newblock In {\em International Conference on Learning Representations}, 2021.

\bibitem[GKM21]{gao2021learning}
Ji~Gao, Amin Karbasi, and Mohammad Mahmoody.
\newblock Learning and certification under instance-targeted poisoning.
\newblock In {\em Uncertainty in Artificial Intelligence}, pages 2135--2145.
  PMLR, 2021.

\bibitem[GM]{gupta2022optimalalgorithmcertifyingmonotone}
Meghal Gupta and Naren~Sarayu Manoj.
\newblock {\em An Optimal Algorithm for Certifying Monotone Functions}, pages
  207--212.

\bibitem[Han07]{hanneke2007bound}
Steve Hanneke.
\newblock A bound on the label complexity of agnostic active learning.
\newblock In {\em Proceedings of the 24th international conference on Machine
  learning}, pages 353--360, 2007.

\bibitem[KL88]{kearns1988learning}
Michael Kearns and Ming Li.
\newblock Learning in the presence of malicious errors.
\newblock In {\em Proceedings of the 20th ACM Symposium on Theory of
  Computing}, pages 267--280, 1988.

\bibitem[LF21]{levine2021deep}
A~Levine and S~Feizi.
\newblock Deep partition aggregation: Provable defense against general
  poisoning attacks.
\newblock In {\em International Conference on Learning Representations (ICLR)},
  2021.

\bibitem[Mil19]{miller2018explanationartificialintelligenceinsights}
Tim Miller.
\newblock Explanation in artificial intelligence: Insights from the social
  sciences.
\newblock {\em Artificial intelligence}, 267:1--38, 2019.

\bibitem[MKSKRJ15]{6868201mozafari}
Mehran Mozaffari-Kermani, Susmita Sur-Kolay, Anand Raghunathan, and Niraj~K.
  Jha.
\newblock Systematic poisoning attacks on and defenses for machine learning in
  healthcare.
\newblock {\em IEEE Journal of Biomedical and Health Informatics},
  19(6):1893--1905, 2015.

\bibitem[MO11]{montejano2011tolerance}
Luis Montejano and Deborah Oliveros.
\newblock Tolerance in helly-type theorems.
\newblock {\em Discrete \& Computational Geometry}, 45(2):348--357, 2011.

\bibitem[Nac18]{separatinghyperplane}
John Nachbar.
\newblock {A Basic Separation Theorem for Cones}.
\newblock
  \url{https://bpb-us-w2.wpmucdn.com/sites.wustl.edu/dist/3/2139/files/2019/09/conesepkkt.pdf},
  2018.

\bibitem[RSG16]{ribeiro2016whyitrustyou}
Marco~Tulio Ribeiro, Sameer Singh, and Carlos Guestrin.
\newblock ''why should i trust you?" explaining the predictions of any
  classifier.
\newblock In {\em Proceedings of the 22nd ACM SIGKDD international conference
  on knowledge discovery and data mining}, pages 1135--1144, 2016.

\bibitem[RSG18]{Ribeiro_Singh_Guestrin_2018}
Marco~Tulio Ribeiro, Sameer Singh, and Carlos Guestrin.
\newblock Anchors: High-precision model-agnostic explanations.
\newblock {\em Proceedings of the AAAI Conference on Artificial Intelligence},
  32(1), Apr. 2018.

\bibitem[SHN{\etalchar{+}}18]{shafahi2018poisonfrogstargetedcleanlabel}
Ali Shafahi, W~Ronny Huang, Mahyar Najibi, Octavian Suciu, Christoph Studer,
  Tudor Dumitras, and Tom Goldstein.
\newblock Poison frogs! targeted clean-label poisoning attacks on neural
  networks.
\newblock {\em Advances in neural information processing systems}, 31, 2018.

\bibitem[SMK{\etalchar{+}}18]{217486}
Octavian Suciu, Radu Marginean, Yigitcan Kaya, Hal~Daume III, and Tudor
  Dumitras.
\newblock When does machine learning {FAIL}? generalized transferability for
  evasion and poisoning attacks.
\newblock In {\em 27th USENIX Security Symposium (USENIX Security 18)}, pages
  1299--1316, Baltimore, MD, August 2018. USENIX Association.

\bibitem[Tuz89]{TUZA198984}
Zsolt Tuza.
\newblock Minimum number of elements representing a set system of given rank.
\newblock {\em Journal of Combinatorial Theory, Series A}, 52(1):84--89, 1989.

\bibitem[ZH16]{Zhou2016InterpretingMV}
Yichen Zhou and Giles Hooker.
\newblock Interpreting models via single tree approximation.
\newblock {\em arXiv: Methodology}, 2016.

\end{thebibliography}
